\documentclass[letterpaper, 10 pt, conference]{ieeeconf}  

\IEEEoverridecommandlockouts                              
\makeatletter
\let\NAT@parse\undefined
\makeatother

\usepackage[dvipsnames]{xcolor}

\newcommand*\linkcolours{ForestGreen}

\usepackage{times}
\usepackage{graphicx}
\usepackage{amssymb}
\usepackage{gensymb}
\usepackage{amsmath}
\usepackage{breakurl}
\usepackage{caption}
\usepackage{subcaption}
\usepackage{cuted}

\usepackage[utf8]{inputenc}

\usepackage[ruled, linesnumbered]{algorithm2e} 
\usepackage[figuresright]{rotating}
\usepackage{booktabs} 
\usepackage{multirow} 

\usepackage{url,hyperref}
\hypersetup{
colorlinks,
linkcolor=\linkcolours,
citecolor=\linkcolours,
filecolor=\linkcolours,
urlcolor=\linkcolours}

\usepackage[labelfont={bf},font=small]{caption}
\usepackage[none]{hyphenat}

\usepackage{mathtools, cuted}

\usepackage[noadjust, nobreak]{cite}
\usepackage{tabularx}
\usepackage{amsmath}
\usepackage{float}
\usepackage{pifont}

\newcolumntype{Y}{>{\centering\arraybackslash}X}
\usepackage[]{placeins}

\usepackage{placeins}

\usepackage{tikz}

\usepackage[framemethod=tikz]{mdframed}
\usepackage{afterpage}
\usepackage{stfloats}
\usepackage{atbegshi}
\usepackage{amsmath}
\usepackage{graphicx}
\usepackage{subcaption}
\usepackage{multirow}

\usepackage{times}  
\usepackage{float}
\usepackage{amsmath}
\usepackage{amsfonts}
\usepackage{helvet}  
\usepackage{courier}  
\newtheorem{theorem}{Theorem}
\newtheorem{definition}{Definition}

\newcommand{\handlethispage}{}
\newcommand{\discardpagesfromhere}{\let\handlethispage\AtBeginShipoutDiscard}
\newcommand{\keeppagesfromhere}{\let\handlethispage\relax}
\AtBeginShipout{\handlethispage}

\usepackage{comment}
\title{\LARGE \bf
Game Theory Meets Statistical Mechanics in Deep Learning Design
}

\author{Djamel BOUCHAFFRA$^{1}$, Fayçal YKHLEF$^{2}$, Bilal FAYE$^{3}$, Hanane AZZAG$^{4}$, Mustapha LEBBAH$^{5}$  \\
	\normalsize djamel.bouchaffra@gmail.com, ykhlef.faycal@gmail.com, faye@lipn.univ-paris13.fr, azzag@univ-paris13.fr, mustapha.lebbah@uvsq.fr
}

\begin{document}

\maketitle
\thispagestyle{empty}
\pagestyle{empty}

\begin{abstract}
We present a novel deep graphical representation that seamlessly merges principles of game theory with laws of statistical mechanics. It performs feature extraction, dimensionality reduction, and pattern classification within a single learning framework. Our approach draws an analogy between neurons in a network and players in a game theory model. Furthermore, each neuron viewed as a classical particle (subject to statistical physics' laws) is mapped to a set of actions representing specific activation value, and neural network layers are conceptualized as games in a sequential cooperative game theory setting. The feed-forward process in deep learning is interpreted as a sequential game, where each game comprises a set of players. During training, neurons are iteratively evaluated and filtered based on their contributions to a payoff function, which is quantified using the Shapley value driven by an energy function. Each set of neurons that significantly contributes to the payoff function forms a strong coalition. These neurons are the only ones permitted to propagate the information forward to the next layers. We applied this methodology to the task of facial age estimation and gender classification. Experimental results demonstrate that our approach outperforms both multi-layer perceptron and convolutional neural network models in terms of efficiency and accuracy.
\end{abstract}


\section{Introduction}
Deep learning (DL) based on graphical representations has proven effective, especially when domain-specific knowledge for feature extraction is limited~\cite{archana2024deep}. For instance, DL models have demonstrated high performance in complex tasks such as medical image classification~\cite{kumar2024medical}~\cite{Bamber2023}, natural language processing, and speech recognition~\cite{chou2024deep}~\cite{kheddar2024automatic}~\cite{10584052}. However, these deep learning models often function as black boxes, delivering impressive results in data classification without providing insights into understanding the model's internal workings as well as unraveling the causal mechanisms underlying their predictions~\cite{10.5555/3202377}~\cite{9363924}. This lack of interpretability and explainability, essentially the ability to comprehend and trace cause and effect within a system, limits their applicability in domains they were not specifically trained for. In terms of their structural design, DL models exhibit several critical limitations: (i) they process information sequentially from one layer to the next without formally evaluating the individual contribution of each neuron, (ii) they have difficulty determining activation levels associated with groups of neurons within a layer, (iii) they struggle to identify the most informative neurons within a layer, often relying on random dropout techniques to mitigate noise and reduce overfitting, and (iv) most of the DL models lack probabilistic measures to express information uncertainty. 

To address these limitations, several approaches exploiting game theory (GT) were proposed in the deep machine learning literature~\cite{maschler2020game} ~\cite{hu2022ore}~\cite{ren2021game} ~\cite{li2021deep}~\cite{yin2024game}. Likewise, ~\cite{hd2017} relied on game theory to improve prediction in ensemble learning. They defined the
pruned ensemble as the minimal winning coalition made of the members that together
exhibit moderate diversity. Moreover, ~\cite{liu2024game} explored unlearnable example attacks using a game-theoretic approach, where the attack is modeled as a nonzero-sum Stackelberg game. ~\cite{meng2023efficient} introduced an efficient approach that leverages a combination of deep learning techniques and game theory to enhance the performance and scalability of solving extensive-form games. These games, characterized by their complex decision-making processes with latent information, pose significant challenges in strategic planning. The research by ~\cite{xing2024deep} addresses the challenge of robots finding optimal paths while avoiding collisions with humans and other robots. Traditional Deep Reinforcement  Learning (DRL) struggles with slow convergence in such complex scenarios. To improve performance, the study introduces a hybrid approach that integrates DRL with game theory. Furthermore, ~\cite{ren2021game} have demonstrated that a deep neural network can be modeled as a non-atomic congestion game, irrespective of whether it is fully connected or only locally connected. Additionally, they have proved that optimizing the weight and bias vectors for a given training set is equivalent to finding the optimal solution for the associated non-atomic congestion game. Other applications of game theory to deep neural networks can be found in ~\cite{hazra2022applications} and ~\cite{Hazra2024}.

A different front emanates from the field of statistical mechanics (SM) has been investigated in order to gain insight into the understanding and optimization of deep learning models ~\cite{tuckerman2023statistical}~\cite{ kollmannsberger2021deep}. ~\cite{PhysRevE.108.014309} applied mean-field theories to analyze the information propagation in neural networks, which helps identify the `edge of chaos' and dynamic isometry conditions for optimal learning and generalization. These theories provide a framework for initializing neural networks in a way that maximizes mutual information, enhancing their performance from the start. In the context of continual learning, statistical mechanics offers insights through the development of variational principles and mean-field potentials.

Our main contribution in this manuscript is fourfold:
\begin{itemize}
\item \textbf{Fusion of GT and SM}: A seamless combination of game theory and statistical mechanics in deep learning design is applied. In this setting that we propose under the name of `NEUROGAME', the collaboration between neurons within layers in a neural network is grounded in game theory driven by statistical mechanics laws. 
\item \textbf{Probabilistic Signal Transmission}: The flow of information, with a Gibbs distribution value, is propagated across layers in the network. 
\item \textbf{Cortical Activation}: A neuronal region of activation within the network is described as a coalition of players—connected neurons cooperating to optimize the payoff function. 
\item \textbf{Information Filtering and Model Regularization}: The coalition with the maximum payoff is deemed the \textit{winning coalition}, and the contribution of each neuron within this coalition is quantified using the Shapley value. Neurons with high contributions form a \textit{strong coalition}, and only these neurons transmit information forward to the next layer. In this very phase, some neurons are dropped out to a achieve a \textit{dynamic model regularization}.
\end{itemize}

\section{Some Basics of Game Theory}
The following definitions are essential to grasp some knowledge about game theory. 
\subsection{Simple Games}
To gain insight into the proposed methodology, we introduce some principles related to game theory, focusing on the concepts of simple games and cooperative sequential games~\cite{von2021game}~\cite{roth2015gets}~\cite{mendelson2024introducing}. A simple game involves a set of $n$ players; a set of strategies $s_i \in \mathcal{S}_i$ (possible actions) associated with each player $i \in \mathcal{N} = \{1,2,\ldots,n\}$, where 
$s=(s_1,s_2,\ldots,s_n) \in  \mathcal{S} = (\mathcal{S}_1\times\mathcal{S}_2\times\ldots\times\mathcal{S}_n)$ is a set of pure strategy profiles;
a set of payoffs (real values) $v_i(s_1,s_2,\ldots,s_n)\in \mathbb{R}$ ($v_i: \mathcal{S} \longrightarrow \mathbb{R}$) assigned to each player $i$ for every possible list of strategy choices—where strategies translate into outcomes and each player has preferences over these outcomes represented by their payoffs—and a level of information or belief, which encompasses what players know and believe about the situation and one another, and what actions they observe before making decisions.
The game is finite if $\mathcal{S}$ is finite.
\subsection{Notion of Simple Coalition} 
We define the concept of coalition and the contribution of each player within this coalition. A simple coalition is a group of players $\mathcal{C} \subset \mathcal{N}$ that cooperate to achieve a common goal. The set $\mathcal{N}$ is often referred to as the grand coalition. Every coalition $\mathcal{C}$ has a set of actions. If the payoffs $v(\mathcal C)$ associated with a coalition $\mathcal C$ are freely redistributed among its members, this condition is known as the Transferable Utility Assumption (TUA).
A coalitional game with transferable utility is a pair $(\mathcal{N}, v)$, in which $\mathcal{N}$ is a finite set of players,~and~$v: 2^\mathcal{N} \longrightarrow \mathbb{R}$~maps each~coalition~$\mathcal{C}$ to a real-valued payoff function $v(\mathcal{C})$ that the coalition members can distribute among themselves. We assume that $v(\emptyset) = 0$. Given a coalitional game $(\mathcal{N}, v)$, the Shapley value associated with player $i \in \mathcal{N}$ is given by:
\begin{equation}
    \phi_i(\mathcal{N}, v) = \frac{1}{n!} \sum_{\mathcal{C} \subseteq \mathcal{N} \setminus \{i\}} |\mathcal{C}|! (n - |\mathcal{C}| - 1)! [v(\mathcal{C} \cup \{i\}) - v(\mathcal{C})].
    \label{shapley}
\end{equation}

The Shapley value expresses the average marginal contribution of player $i$, averaging over all different coalitions with respect to which the grand coalition can be built starting from the empty one~\cite{maschler2020game}.

\section{NEUROGAME: Game Theory Meets Statistical Mechanics}
We present in this section the analogy between conventional 
DL and NEUROGAME, as well as the description of all the 
components needed to fully comprehend how this proposed deep learning model operates. 
\subsection{Comparison between Conventional Deep Learning and NEUROGAME}
we make the following correspondence between cooperative 
game theory and deep learning representation:
\begin{enumerate}
\item A layer of a deep neural network represents a game. 
\item A neuron in a layer of a deep neural network represents a player of the game. Neurons are viewed as particles interacting via statistical mechanics laws.
\item Each neuron is mapped to a set of actions representing its 
current state (a specific interval of neuron activation 
values). This set of actions acts as its strategy. 
\item An input to the deep neural network structure corresponds 
 to the information (or observation within the environment) 
 that is available at any time of the game. In our setting, an 
 input is a 2D image. 
 \item A neuronal region depicts a group of connected neurons $(s_1,\ldots,s_n)$ that are located within a certain neighborhood in the cerebral cortex; it constitutes a simple coalition of players. 
 \item A payoff $v_i(s_1
, \ldots, s_n)$ assigned to this coalition expresses the worth of the actions exhibited by all players forming this 
 coalition. This function is conveyed through the energy 
 function assigned to a tuple of activations of neurons. This tuple is called \textit{a configuration state of the coalition}. The coalitions with high payoffs are sought: 
 They represent \textit{the winning coalitions}. The payoff 
 function reflects the quality level of the information available. 
 \item The contribution of a neuron within each winning 
 coalition is expressed by its Shapley value expressed through equation~\eqref{shapley}. Neurons with 
 high Shapley values are members of strong coalitions. 
 Only strong coalitions, extracted from the winning coalitions, are permitted to forward the flow of information from one layer to the next.
 \end{enumerate}
%
%
\subsection{Computation of a Coalition Payoff}
This section describes the relationship between the energy 
function, the Gibbs distribution and the payoff (also known 
as utility) function assigned to a coalition. The computation 
of all three functions requires a definition of a neighborhood 
system between neurons that compose a coalition.
\subsubsection{Neurons Neighborhood System:} For the sake of illustration and without loss of generality, we 
focus on neighbors of a neuron within a (3,3) neuronal grid. 
A $(3,3)$ neighborhood system $\mathcal{H}$ of a neuron 
located at coordinates $(i,j)$ is the set $\{(i-1,j), (i-1,j+1), (i,j+1), 
(i+1,j+1), (i+1,j), (i+1,j-1), (i,j-1), (i-1,j-1)\}$. 

This neighborhood system is needed during the clique
structure used by the energy function. 
\begin{definition} A set of random variables is a Gibbs random 
field (GRF) on a set $\Omega$ with respect to a neighborhood 
system $\mathcal{H}$ if and only if its configuration obeys a Gibbs 
distribution. 
\end{definition}
We now introduce the notion of configuration state that is 
needed in the evaluation of the energy and payoff functions. 
\begin{definition} A configuration state assigned to a simple 
coalition is a sequence of activation values of neurons that 
form this coalition. 
\end{definition}
\noindent This configuration state is denoted: 
$\omega_i = (a_{s_1}^i,\ldots, a_{s_n}^i)$
where: $a_{s_j}^i$ is the neuron activation value at location $s_j$ in the coalition $i$ and $n$ is the size of the coalition. 
\subsubsection{Gibbs Distribution of a Configuration State:} During a regression or classification task, we aim for the activation of neurons to progressively increase from the first layer to the last layer in a deep neural network. This behavior is compatible with the energy minimization 
principle. The Gibbs (or Boltzmann) distribution relies on the energy function assigned to the $i$-th configuration state.
\begin{definition} The Gibbs distribution function is defined as: 
\begin{equation}\label{gibbs}
P(\omega_i,T)=\frac{1}{Q}e^{\frac{-E(\omega_i)}{k_B\times T}}=\frac{e^{\frac{-E(\omega_i)}{k_B\times T}}}{\sum_{j=1}^{j=M}e^{\frac{-E(\omega_j)}{k_B\times T}}},
\end{equation}
\end{definition}
\begin{itemize}
\item $P(\omega_i,T)$
 is the probability of the $i$-th configuration state at temperature $T$, 
\item $E(\omega_i)$ is the energy of the $i$-th configuration state, 
\item $k_B$ represents the Boltzmann constant $(k_B\approx 1.38\times10^{-23})$,
\item $T$ is the temperature of the system, 
\item $M$ denotes the number of all configuration states 
associated to all simple coalitions within a layer, 
\item $Q$ is the canonical partition function (normalizing 
factor).
\end{itemize}
This distribution shows that configuration states with lower 
energy will always be assigned a higher probability of being 
occupied than those with higher energy. However, the 
energy assigned to a configuration state is defined via a 
potential function expressed through the Ising model using 
bonding strengths (synaptic links) between neurons in a 
lattice structure. This energy, which is a Hamiltonian 
function, is therefore expressed as: 
\begin{equation}\label{energy}
E(\omega)=\sum_{<p,q>}b_{pq}\times\left(\frac{1}{a_p\times a_q}\right)+\sum_pf_p\times\left(\frac{1}{a_p}\right),
\end{equation} 
\begin{itemize}
\item $b_{pq}$ is the bonding strength between two neighbor 
neurons $p$ and $q$, 
\item $f_q$ is the external magnetic field interacting with the 
lattice, 
\item $a_p$, and $a_q$ are non-zero activation values assigned 
to neuron $p$ and $q$, respectively,
\item $<p,q>$ is a pair of neighbor neurons. 
\end{itemize}
If we set $f_p=\alpha$ and $b_{pq}= \beta$, therefore the Ising model 
expressed via equation~\ref{energy} can be rewritten as follows: 
\begin{equation}\label{energy_bis}
E(\omega)=\alpha\sum_p\left(\frac{1}{a_p}\right)+\beta\sum_{q\in \mathcal{H}(p)}\left(\frac{1} {a_p\times a_q}\right), \forall p
\end{equation}
where $\mathcal{H}(p)$ is a $(m,n)$ neighborhood system. The second 
summation is over pairs of neighboring neurons. The energy decreases when the activation values in a configuration state are high. \textit{In other 
words, a smaller energy means a higher neuronal activation.} 
However, we consider the temperature 
$T$ as dependent on the iteration number 
$i$ during the training of NEUROGAME. It is expressed as follows:
\begin{equation}\label{temperature}
T(i)=\frac{c\times 10^{23}}{\ln{(1+i)}},
\end{equation}
where the numerator 
is a large constant value that ensures a 
high temperature at initialization. Therefore, using equation~\ref{gibbs}, one can compute the Gibbs distribution $P(\omega_i,T)$ assigned 
to each configuration state. 
\subsubsection{Generation of Configurations States:} In order to compute the Gibbs distribution, one 
has to estimate the normalizing term that requires $M$ 
configuration states. The set of configuration states 
contained in one layer is built using a grouping (set of 
neurons acting together) containing neurons that are close to 
each other. An element of this grouping can be a 4$\times$4 (or 
5$\times$5) grid of neurons. \textit{A simple coalition in a layer is composed of 
neurons that are nearby with respect to a distance measure. }
We generate through this partitioning process $M$ 
configuration states with different levels of neuron 
activations. Moreover, each configuration state is assigned a 
Gibbs distribution value. 
\subsubsection{Layer Neighborhood System:} We show in this step how a neighborhood system (a lattice structure)
can be constructed in 
order to compute the energy associated to the Ising model. 
A neighborhood system is based on a metric (or distance) 
between neurons of a layer. 
This set 
of neighbors associated to neuron $(i,j)$ is composed of the 
sites $\{(i,j-1), (i-1,j), (i,j+1), (i+1,j)\}$. 
\subsubsection{The Payoff Function:} Since we are in the context of a collaborative game theory, 
therefore, the contribution of a group of players should 
induce a higher payoff than the one incurred by a single 
player within a simple coalition. Furthermore, a maximum payoff value should be assigned a minimum energy value. Using Boltzmann’s distribution, this minimum energy value 
corresponds to a maximum Gibbs distribution value. We now define the payoff as being proportional to the Gibbs distribution:
\begin{definition} The payoff function assigned to a simple coalition is expressed as follows:
\begin{equation}\label{payoff}
\text{Payoff}(\omega_i,T) =\ln\left(\frac{k_1}{1-P(\omega_i,T)}\right),
\end{equation}
where $k_1$ is a positive control parameter and the natural 
logarithm is applied to smooth this function.
\end{definition}
One can notice 
that a high payoff corresponds to a low neuronal energy 
value. Neurons are supposed to behave as microscopic 
physical particles interacting seamlessly. The configuration 
state $\omega^*$ with the maximum Payoff value is assigned the 
winning coalition among all simple coalition associated to 
the $M$ possible configuration states.
\begin{definition}
A configuration state $\omega^*$ with a maximum 
Payoff value is associated to a winning coalition among all 
possible simple coalitions. 
\end{definition}
The Payoff value represents the worth of the winning 
coalition. It tallies the total expected sum of payoffs the 
members of this coalition can gain by cooperating. 
However, instead of considering only one winning coalition, 
a set of $p$ winning coalitions derived from $p$ top choices of 
payoff values are considered. 
\subsubsection{The Concept of Strong Coalition:} The payoff value assigned to a configuration state of is needed during the Shapley value computation. This payoff 
corresponds to the utility function $v$ used in the Shapley function expressed by equation~\ref{shapley}. This payoff function requires the 
computation of: $v(\mathcal{C}\cup\{i\})-v(\mathcal{C})$, which is the leading term 
in the Shapley value computation, associated to player $i$, denoted 
$\phi_i(\mathcal{N},v= \text{Payoff})$. This leading term corresponds 
to: 
\begin{equation}
\text{Payoff}(\mathcal{C}\cup\{i\})-\text{Payoff}(\mathcal{C}), \forall \mathcal{C}{\subseteq(\mathcal{N}\setminus\{i\})}.
\end{equation}
The summation used in equation~\ref{shapley} consists in extracting 
all subsets $\mathcal{C}$ from the simple coalition $\mathcal{N}$ (set of players) that 
do not contain player $i$. The number of these subsets is $2^{\mathcal{N}-1}$. 
However, among all subsets, only those subsets whose 
cardinalities are greater or equal than $2$ are considered, since 
the singletons do not form coalitions. 
Once the winning coalition is identified, members of this 
coalition who contributed most to the payoff are maintained; 
the other members with low contributions are dropped out. This regularization technique that is not based on randomness represents one key feature of novelty exhibited by NEUROGAME. 
\begin{definition} A strong coalition is composed of all neurons whose Shapley values are greater than a threshold value $\rho$.
\end{definition}
Neurons with a high payoff (or high Gibbs distribution) are those that form the strong coalition. The threshold $\rho$ is 
dynamic; it involves the contribution of each neuron forming a coalition within a network layer and the iteration 
number $i$: it is expressed specifically via the following function:
\begin{equation}
\rho(S_{c_j}^i,i)=Q_1(S_{c_j}^i) \times \ln(1+i).
\end{equation}
The function `$\ln$' represents the natural logarithm, while $Q_1(S_{c_j}^i)$ denotes the first quartile of the set $S$ of Shapley values (sorted in ascending order) assigned to the set of neurons forming a winning coalition $c_j$ for $j=1, \ldots, p$, with each coalition being of size $n$. If $n$ is the number of these values, therefore, this first quartile is equal to $(n+1)/4$; it 
indicates that $25\%$ of the data are below this point.
\begin{theorem}[Shapley Threshold Behavior]
For large values of $n$ (coalition size) and $i$ (iteration number during training), the function $\rho(S_{c_j}^i,i)$  will tend to increase, with the growth rate influenced by $ln(1+i)$.
\end{theorem}
\begin{proof}
If $i$ increases during training, the natural logarithm function $ln(i)$ grows without bound, but it does so very slowly compared to linear functions. Therefore, $ln(1+i)$ will continue to increase, but at a gradually slowing, logarithmic rate. However, $Q_1(S_{c_j}^i)$ depends on the Shapley values distribution. As we increase the coalition size $n$, the value of $Q_1(S_{c_j}^i)$ does not necessarily increase. It reflects the position within the ordered data rather than growing unbounded. Finally, the combined effect on $\rho(S_{c_j}^i,i)$ will be affected by the product of these two functions: $Q_1(S_{c_j}^i)$ and $ln(1+i)$. In conclusion, the primary driver of the behavior of the threshold $\rho(S_{c_j}^i,i)$ for large $i$ and $n$ will be $ln(1+i)$.
\end{proof}
\noindent It is worth noting that as NEUROGAME learns, the selected coalitions grow progressively stronger.
\subsection{The Different Phases in NEUROGAME}
The following sequence of operations describes NEUROGAME:
\begin{enumerate}
\item The input is a colored (or grey-level) image with its
three colors components \textit{red}, \textit{green}, \textit{blue} $(n_c = 3)$, with
a dimension equal to (n×n) for each color: $(n \times n \times n_c)$.
\item A convolution operation with three filters $F_1$, $F_2$ and $F_3$,
each one with a dimension equal to $f \times f$ is subsequently
applied, to each color $(f\times f \times n_c)$. An arithmetic mean
value is computed for each element of the three colored
matrices after convolution.
\item The results of the convolution between filters and the
image is represented by three feature map matrices with dimension $(n-f+1)\times(n-f+1)\times n_c$.
\item An activation function is applied to the product of the
feature map matrices and the first weight matrix $W_1$,
and the result is stored as the three activation map
matrices with a predefined dimension $(l\times l)$.
\item Generation of $M$ simple coalitions within each of the the
three activation map matrices of the first layer. The value of $M$ is equal to the dimension of a layer divided by $n$. Therefore, each simple coalition has a size of $n$ neurons and is assigned a configuration state (activation values).
\item Computation of the energy value for each simple coalition (configuration state) within an activation map matrix. The set of energy values within an activation map
of the first layer forms an energy map. Therefore,
we obtained three energy maps.
\item Selection of $p$-top choices simple coalitions given their payoff values. They are the $p$ winning
coalitions of each energy map.
\item Extraction of the strong coalitions amongst the winning
ones. Neurons of the winning coalitions whose Shapley
values exceed the threshold value $\rho(S_{c_j}^l,i)$ are maintained
and neurons with Shapley values that fall under this
threshold value are removed.
\item This entire process continues during the first training
iteration until reaching the last layer $k$. The activation values of the strong
coalitions corresponding to the three energy maps are
concatenated to form a feature vector assigned to the
input image.
\item This feature vector is subsequently fed to a fully connected neural network with $k$ hidden layers.
\item The Softmax operation is applied for the evaluation of
the loss function during training.
\item All weights are updated using the opposite direction of
the gradient of the loss function.
\end{enumerate}
Figure~\ref{fig1} illustrates the NEUROGAME training procedure 
when the observation input is an image and the number of 
labels for a classification task is four $(C_1, C_2, C_3, C_4)$.
\begin{figure*}[ht]
    \centering
    \includegraphics[width=0.85\textwidth]{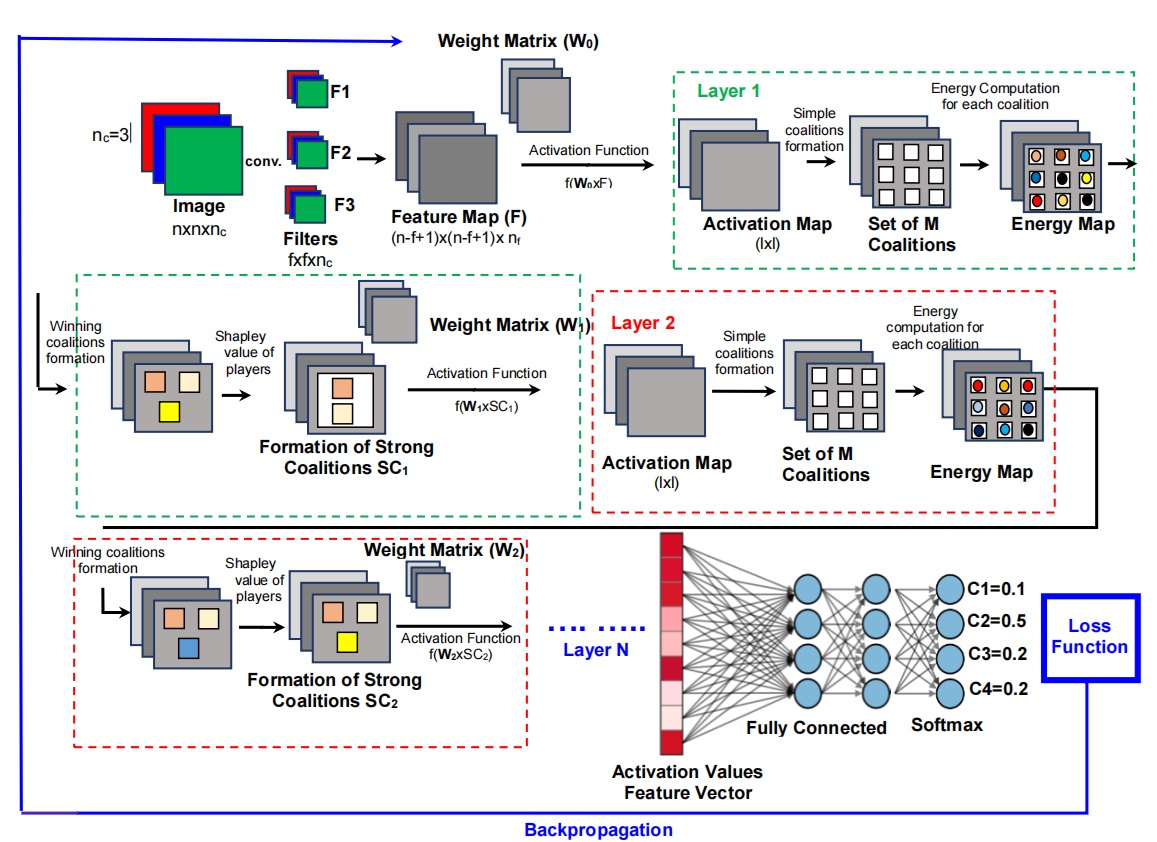}
    \caption{\footnotesize{The training procedure of NEUROGAME showing the passage from \textit{M simple coalitions} to \textit{p winning coalitions} and then to \textit{p strong coalitions} generated via the Shapley filtering process. The computation of the strong coalitions (integrated into a fully connected neural network) is repeated across all $k$ layers until NEUROGAME converges. The feature vector extracted at this convergence point is composed of activation values of the last optimal strong coalitions.}} 
    \label{fig1}
\end{figure*}

\subsection{NEUROGAME Layers and Information Propagation} 
In this section, we show how a layer of NEUROGAME is built. We also describe how the nformative signals are communicated to the next layer during training of the entire deep neural network.
\subsubsection{NEUROGAME Layer:} A layer in this proposed deep neural network is composed of five components: (i) activation maps, (ii) a set of $M$ coalitions, (iii) a set of energy maps, (iv) a set of winning coalitions, and (v) a set of strong coalitions (refer to Figure~\ref{fig1}).

\subsubsection{Transmission of the Information:}
The most informative signals generated from neurons pertaining to the strong coalitions 
 (those that passed the $\rho$ test) of layer $l$ are forwarded to neurons of layer $(l+1)$. In fact, these signals represent the image by the activation function of the product of two quantities: (i) The 
activation value of a neuron within a strong coalition in layer $l$, and (ii) the weight (synaptic link) that connects this neuron to a specific neuron of layer $(l+1)$. These two quantities are the ones involved during a forward transmission of information during NEUROGAME training-based on backpropagation. 

\section{Experiment}
\indent To demonstrate the effectiveness of the proposed methodology, we have performed two different classification tasks: 1) gender classification, and 2) simultaneous age and gender classification.
\subsection{Datasets and Architecture}
To assess NEUROGAME's performance, we used two benchmarked datasets designed for distinct classification tasks: \textbf{CelebA (CelebFaces Attributes)}~\cite{liu2015deep} dataset for gender classification and \textbf{UTKFace dataset}~\cite{utkface} for age and gender classification concurrently. We now present the architectures of the two baseline models for gender classification and simultaneous age and gender classification, alongside comparisons with our proposed NEUROGAME method.\newline

\noindent\textbf{Gender Classification:  Multi-Layer Perceptron (MLP):}
\begin{itemize}
    \item Input: Images of size (64, 64, 3).
    \item Layers: Flattened input followed by dense layers (256 units ReLU, Batch-Normalization, Dropout; 128 units ReLU, Batch-Normalization, Dropout; 64 units ReLU, Batch-Normalization, Dropout).
    \item Output: Single unit with sigmoid activation for binary classification.
\end{itemize}
For comparison with NEUROGAME, a single NEUROGAME layer is added with a coalition size of (2,2) and a top-p value of 0.85.\newline

\noindent\textbf{Simultaneous Gender and Age Classification: Convolutional Neural Network (CNN):}
\begin{itemize}
    \item Input: Images of size (128, 128, 1).
    \item Layers: Four Conv2D layers (32, 64, 128, 256 filters with (3, 3) kernels and ReLU activation), followed by max-pooling (2, 2).
    \item Flatten and two dense layers (256 units each, ReLU activation), each followed by a Dropout layer.
    \item Outputs: Gender (sigmoid activation), Age (ReLU activation). 
\end{itemize}
For comparison, Conv2D layers are replaced with NEUROGAME layers (top-p=0.85, coalition size=(2,2)) to evaluate NEUROGAME's performance in classification tasks. In all NEUROGAME models, we applied a Convolutional layer with three filters to generate feature maps.
\subsection{Hyperparameter Tuning for NEUROGAME}
To determine the optimal hyperparameter values for NEUROGAME, extensive experimentation was carried out. The results showed that the most effective configuration was achieved by setting 
$\alpha$ to $0$ and 
$\beta$ to $1$ (refer to Equation~\ref{energy_bis}). With $\alpha=0$, the model’s energy is determined exclusively through the interactions between neighboring neurons, thereby simplifying the system by excluding the contributions of individual neurons. Setting $\beta=1$ preserves the original form of neighbor interactions without any additional weighting. For temperature estimation, a value of $c=1$ was found to be optimal (refer to Equation~\ref{temperature}), while $k_1=1$ was determined to be the best setting for the payoff calculation (refer to Equation~\ref{payoff}).
\subsection{Gender Classification}
\label{sec:gender}
For gender classification, we applied data augmentation through random cropping and horizontal flipping to increase training diversity and model robustness. The images were normalized by scaling pixel values to [0, 1] and dividing by 255. Models were trained with a batch size of $64$ using Adam optimizer~\cite{kingma2014adam} and binary cross-entropy loss. Figure~\ref{fig6} shows that NEUROGAME achieved more effective reduction in loss and better generalization, with a test loss of $0.2645$ compared to MLP's $0.4335$, as highlighted in Table~\ref{tab:performance_comparison}, demonstrating NEUROGAME's superior performance.
\begin{table}[H]
    \centering
    \begin{tabular}{|c|c|c|}
        \hline
        \textbf{Model} & \textbf{Test Loss} & \textbf{Test Accuracy (\%)} \\
        \hline
        MLP & 0.4335 & 80.19 \\
        \hline
        NEUROGAME & \textbf{0.2645} & \textbf{88.26} \\
        \hline
    \end{tabular}
    \caption{Test performance comparison between MLP and NEUROGAME models on CelebA dataset.}
    \label{tab:performance_comparison}
\end{table}
Furthermore, the test accuracy of NEUROGAME is $88.26\%$, substantially higher than MLP model's test accuracy of $80.19\%$. This improvement in accuracy demonstrates NEUROGAME's superior capability in generalizing from the training data to unseen data, confirming that the incorporation of NEUROGAME's specialized layer results in enhanced model performance and reliability. To further assess the efficacy of NEUROGAME, we  expanded our investigation to include in the next section a more intricate task: simultaneous classification of age and gender.
\subsection{Simultaneous Age and Gender Classification}
\label{sec:gender_age}
In age and gender classification, we use two classifiers: the first focuses on gender with binary cross-entropy loss and accuracy as the metric, while the second predicts age as a continuous value using mean absolute error as the loss function. During inference, a correct age class prediction is considered a success. Both models are optimized using the Adam optimizer with a batch size of $32$ and trained for $100$ epochs without data augmentation. This setup thoroughly assesses the performance and robustness of CNN and NEUROGAME in age and gender classification.\newline
\begin{figure}[ht]
    \centering
    \includegraphics[width=0.5\textwidth]{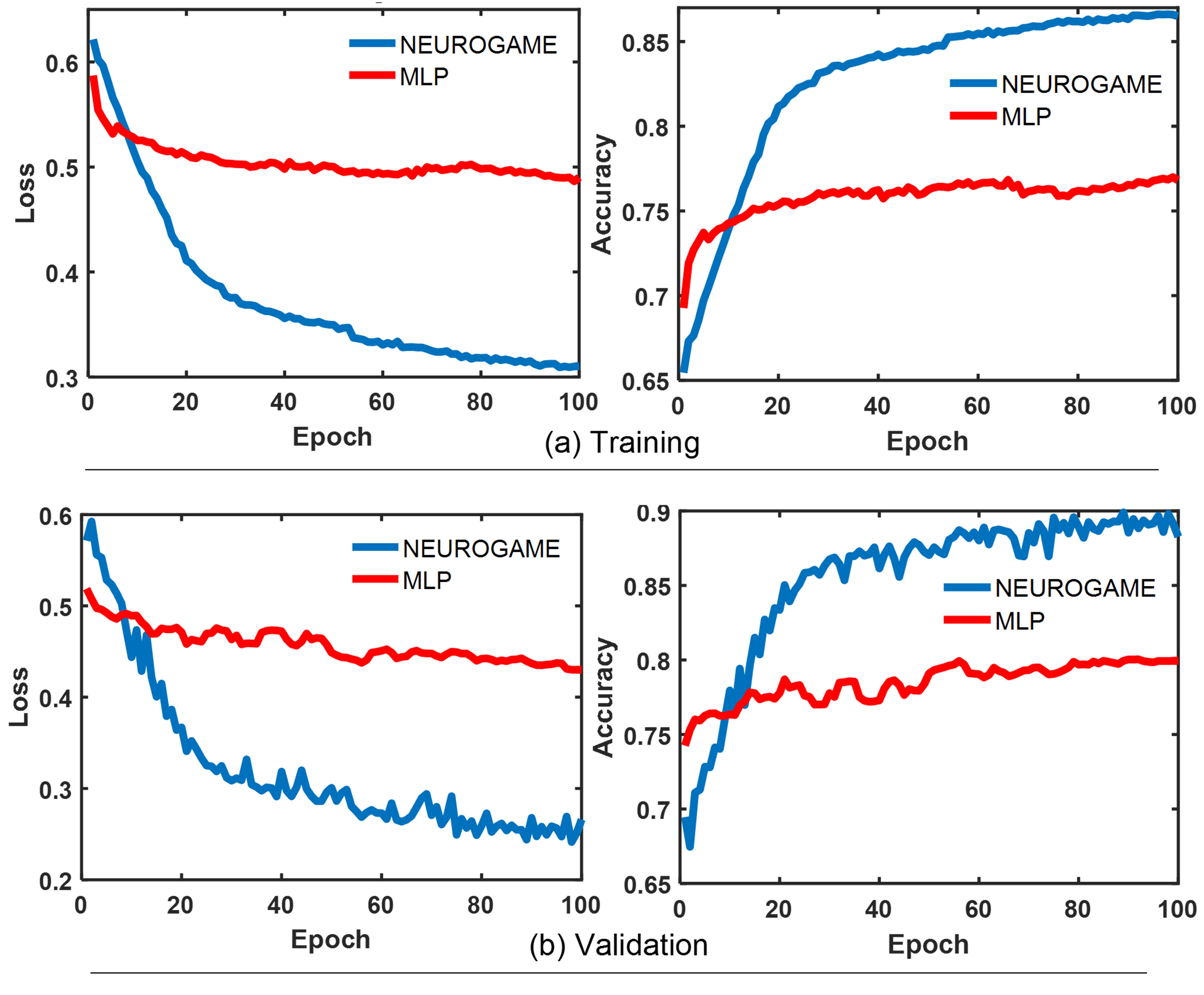}
    \caption{\footnotesize{ Comparison of training and validation losses and accuracies between MLP and NEUROGAME models. NEUROGAME shows better generalization performance, as evidenced by the lower validation loss and improved validation metrics.}} 
    \label{fig6}
\end{figure}
\begin{figure}[ht]
    \centering
    \includegraphics[width=0.5\textwidth]{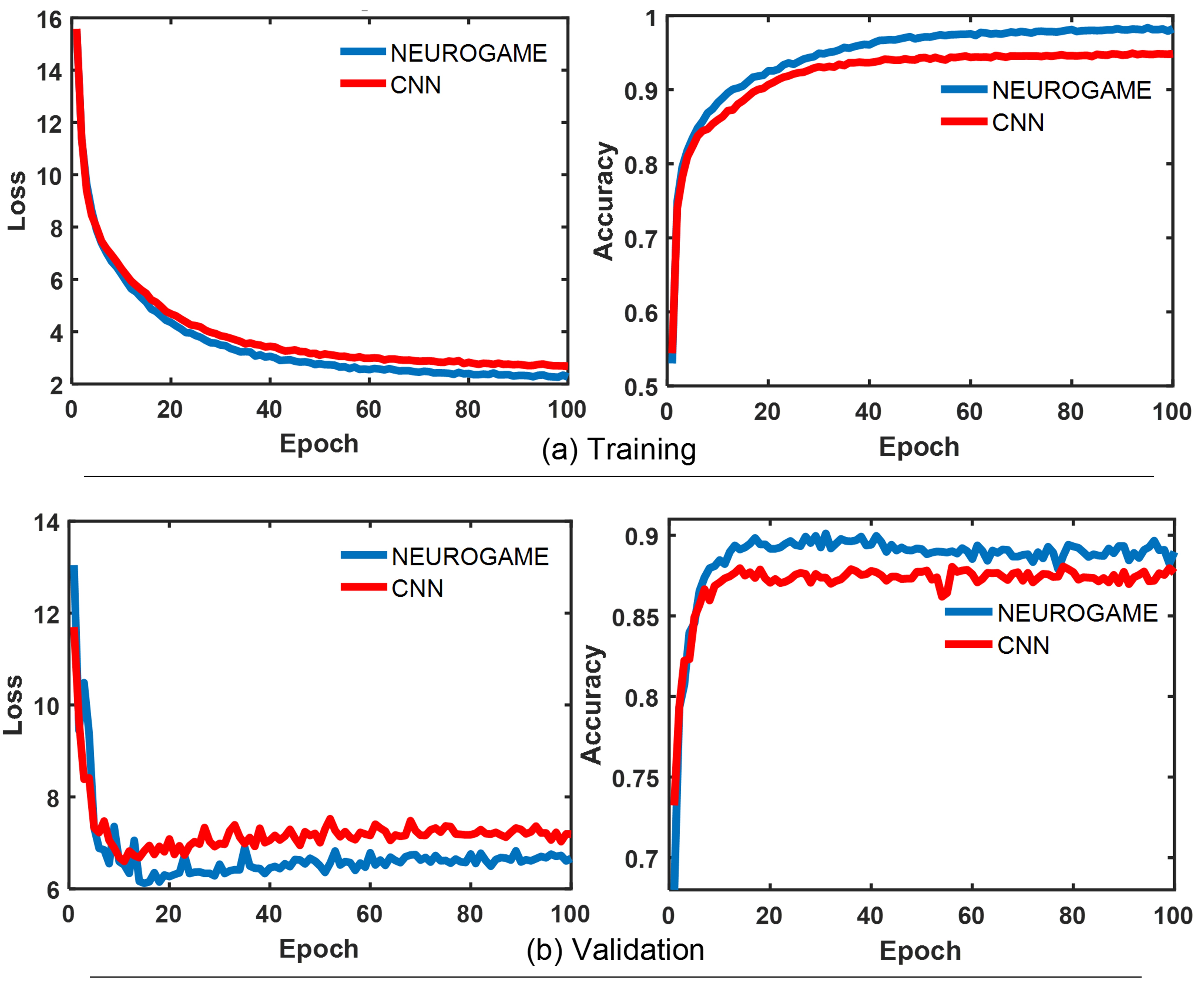}
    \caption{\footnotesize{ Comparison of training and validation performance between CNN and NEUROGAME models for gender classification.}} 
    \label{fig7}
\end{figure}
\begin{figure}[ht]
    \centering
    \includegraphics[width=0.5\textwidth]{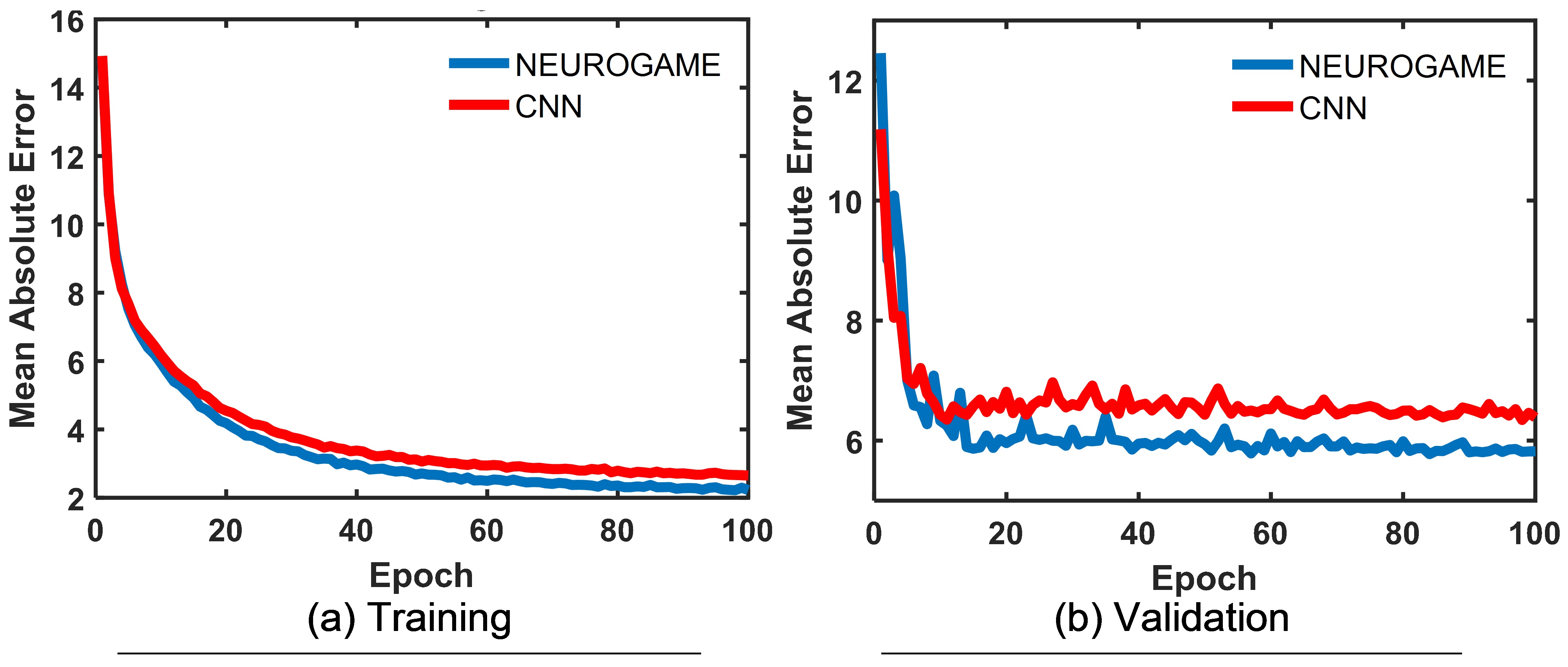}
    \caption{\footnotesize{Comparison of training and validation performance between CNN and NEUROGAME for age classification.}} 
    \label{fig8}
\end{figure}
Figures~\ref{fig7} and~\ref{fig8} compare CNN and NEUROGAME in terms of training and validation performance for gender and age classification. NEUROGAME consistently outperforms CNN, demonstrating superior generalization with lower validation loss and better metrics. Precision was computed for each model across gender, age, and combined classifications. NEUROGAME, with fewer parameters, shows higher average precision than CNN. As shown in Table~\ref{tab:sample}, NEUROGAME achieves higher precision across all age groups in gender classification and generally leads in age classification. This was validated on UTKFace test set, where NEUROGAME maintained higher precision, especially in younger and middle age categories, highlighting its robustness in multitask learning.

\begin{table}[h]
    \centering
    \resizebox{\columnwidth}{!}{%
        \begin{tabular}{c|cc|cc|cc}
            \toprule
            \multirow{2}{*}{Class} & \multicolumn{2}{c|}{Gender} & \multicolumn{2}{c|}{Age} & \multicolumn{2}{c}{Gender and Age} \\
            \cmidrule{2-7}
                                   & CNN & NEUROGAME & CNN & NEUROGAME & CNN & NEUROGAME \\
            \midrule
            {[}0, 2{]}   & 89.61  & 91.63  & 85.60  & 69.96  & 79.31  & 64.11  \\
            {[}3, 6{]}   & 95.35  & 96.63  & 79.07  & 81.43  & 76.37  & 77.61  \\
            {[}7, 12{]}  & 93.92  & 96.03  & 72.68  & 74.75  & 71.27  & 73.44  \\
            {[}13, 17{]} & 96.00  & 95.03  & 57.92  & 61.09  & 57.43  & 64.40  \\
            {[}18, 22{]} & 97.62  & 98.04  & 60.93  & 64.18  & 60.07  & 63.46  \\
            {[}23, 26{]} & 98.60  & 98.75  & 53.25  & 58.25  & 52.70  & 57.73  \\
            {[}27, 33{]} & 98.41  & 98.46  & 74.45  & 71.36  & 73.69  & 70.63  \\
            {[}34, 44{]} & 98.79  & 98.65  & 76.52  & 70.97  & 75.97  & 70.39  \\
            {[}45, 59{]} & 98.33  & 97.81  & 77.28  & 73.00  & 76.87  & 72.41  \\
            {[}60, 69{]} & 98.32  & 98.93  & 64.19  & 63.96  & 64.04  & 63.73  \\
            {[}70, 79{]} & 95.00  & 98.05  & 58.80  & 62.66  & 58.23  & 62.09  \\
            {[}80, 89{]} & 98.97  & 97.88  & 50.79  & 59.13  & 50.40  & 57.54  \\
            {[}90, 99{]} & 95.35  & 96.30  & 36.50  & 55.47  & 35.04  & 54.74  \\
            {[}100, 116{]}& 100.00 & 100.00 & 46.88  & 53.13  & 46.88  & 53.13  \\
            \bottomrule
        \end{tabular}%
    }
    \caption{precisions (\%) of CNN and NEUROGAME on UTKFace test set across age and gender categories.}
    \label{tab:sample}
\end{table}
Due to limited data in this age group, CNN model predicts an age of $93$ years, while NEUROGAME predicts $101$ years, closer to the ground truth, showing superior generalization. This indicates NEUROGAME's better handling of sparse data compared to CNN. The image was randomly selected, underscoring NEUROGAME's robustness and reliability.
\indent The two experiments highlight the effectiveness of NEUROGAME, especially in classification tasks, when compared to well-established ML models. Indeed, NEUROGAME has outperformed both MLP and CNN models in gender classification as well as in the simultaneous classification of gender and age.

\section{Conclusion and Perspectives}
We have developed a novel DL architecture, NEUROGAME, which integrates game theory and statistical physics principles. This allows neurons in the same layer to collaborate using the Shapley value function to assign contribution scores and perform controlled dropout, reducing overfitting. This Shapley-based regularization enhances network robustness and provides transparency within the architecture, functioning as a \textit{glass-box framework}.

Comparative studies show NEUROGAME outperforms MLP and CNN in gender and joint gender-age classification, showing better generalization and accuracy. This research signals a paradigm shift in deep learning, paving the way for more interpretable, efficient, and effective neural networks. As a perspective, we will explore \textit{the Banzhaf power index} to assess the influence of neuronal states in prediction tasks, potentially improving model generalization further.

\bibliographystyle{ieeetr}
\bibliography{bibliography}

\clearpage

\end{document}